%% file: icml2023.tex
\documentclass{article}

\usepackage{microtype}
\usepackage{graphicx}
\usepackage{subfigure}
\usepackage{booktabs} 

\usepackage{hyperref}

\usepackage[accepted]{icml2023}

\usepackage{amsmath}
\usepackage{amssymb}
\usepackage{mathtools}
\usepackage{amsthm}

\usepackage[capitalize,noabbrev]{cleveref}
\usepackage{comment}
\usepackage{epigraph}
\input{math_commands.tex}

\usepackage{wrapfig}
\usepackage{braket}
\usepackage{multirow}
\usepackage{booktabs} 
\usepackage{hyperref}
\usepackage{url}
\usepackage{color,xcolor}
\usepackage{array}
\usepackage{algorithm}
\usepackage{todonotes}

\usepackage{amsthm}

\theoremstyle{plain}
\newtheorem{theorem}{Theorem}[section]

\theoremstyle{definition}

\theoremstyle{remark}

\newtheorem{claim}[theorem]{Claim}

\begin{document}

\twocolumn[
\icmltitle{Language Modeling Using Tensor Trains}






\icmlsetsymbol{equal}{*}

\begin{icmlauthorlist}
\icmlauthor{Zhan Su}{equal,yyy}
\icmlauthor{Yuqin Zhou}{equal,yyy}
\icmlauthor{Fengran Mo}{comp}
\icmlauthor{Jakob Grue Simonsen}{yyy}
\end{icmlauthorlist}

\icmlaffiliation{yyy}{Department of Computer Science, University of Copenhagen, Copenhagen, Denmark}
\icmlaffiliation{comp}{Department of Computer Science and Operations Research, University of Montreal, Quebec, Canada}

\icmlcorrespondingauthor{Zhan Su}{zhan.su@di.ku.dk}
\icmlcorrespondingauthor{Jakob Grue Simonsen}{simonsen@di.ku.dk}


\vskip 0.3in
]



\printAffiliationsAndNotice{\icmlEqualContribution} 

\begin{abstract}
We propose a novel tensor network language model based on the simplest tensor network (i.e., tensor trains), called `Tensor Train Language Model' (TTLM). TTLM represents sentences in an exponential space constructed by the tensor product of words, but computing the probabilities of sentences in a low-dimensional fashion. We demonstrate that the architectures of  Second-order RNNs, Recurrent Arithmetic Circuits (RACs), and Multiplicative Integration RNNs are, essentially, special cases of TTLM. Experimental evaluations on real language modeling tasks show that the proposed variants of TTLM (i.e., TTLM-Large and TTLM-Tiny) outperform the vanilla Recurrent Neural Networks (RNNs) with low-scale of hidden units. \footnote{The code is available at \url{https://github.com/shuishen112/tensortrainlm}.}
\end{abstract}

\section{Introduction}



Human languages, like many biological systems, including families of proteins, genomes, and neurons in the brain, have significant long-range correlations that decay with a power law \citep{tagliazucchi2012criticality, mora2011biological}.  
Current network models like LSTMs \cite{hochreiter1997long} are hard to match long-range and higher-order statistics of natural languages \citep{lin2016critical}.

Recently, researchers have turned to tensor network language modeling,  which contains models that exhibit correlation functions that decay with the power law \citep{pestun2017tensor,pestun2017language,miller2021tensor}. Tensor networks are, roughly, decompositions of large tensors into sets of smaller tensors and have been employed in physics, mathematics, and machine learning \citep{cohen2016expressive}. However, the so-called `tensor network language model' is either a concept that needs to be proved practically \cite{pestun2017tensor} or unsuitable in real-world language modeling tasks \citep{miller2021tensor} due to their way of modeling probabilities.  Towards making tensor network language modeling practical, we make the first step to applying it to real language modeling datasets.

As proof-of-concept work, we derive a Tensor Train Language Model (TTLM) (the simplest tensor network).  Technically, we represent a sentence based on the exponential semantic space constructed by the tensor product of word representations. The probability of the sentence is defined by the inner product of two high-dimensional tensors: the input $\Phi(X)$ and the global 
coefficients $\mathcal{A}$, and decomposed into conditional probabilities.  

Under the framework of TTLM,  we propose two variants:  TTLM-Tiny and TTLM-Large. Also, we clarify the relationship between the proposed TTLM and a series of Recurrent Neural Networks (RNNs) (i.e., Second-order RNNs \citep{goudreau1994first},  Recurrent Arithmetic Circuits (RACs) \citep{levine2018benefits}, and Multiplicative Integration RNNs (MI-RNNs) \citep{wu_multiplicative_2016}). These connections open a new eye to understanding RNNs and give some natural implementations for TTLM.

We benchmark these TTLM variants and analyze the difference in their working mechanism and behaviors. Experimental results on the language modeling task show that our TTLM variants could outperform than Vanilla-RNNs under the same training setting. These demonstrate the feasibility of TTLM.

The main contributions of our work can be summarized as follows:
\begin{itemize}
    \item We propose a novel Tensor Train Language Model, as an illustration of how tensor networks can be applied to real-world language modeling datasets. 
    \item We propose two novel TTLM variants, TTLM-Large and TTLM-Tiny, and theoretically demonstrate the relationship between  TTLM and a series of existing RNNs.
    \item Compared to Vanilla-RNNs on WikiText-2 and PTB datasets, TTLM-Large reduces perplexity by 14.3 and 16.0, respectively, and TTLM-Tiny reduces perplexity by 1.7 and 8.5, respectively.
\end{itemize}

\section{Related Work}

Previous studies on tensor networks in machine learning have mainly been devoted to analyzing the theoretical properties of neural networks. A better understanding of feed-forward, convolutional and recurrent architectures has been gained, including compression parameters \citep{novikov2015tensorizing}, expressive power \citep{cohen2016expressive, cohen2016convolutional, khrulkov2018expressive}, and depth efficiency for long-term memory \citep{levine2018benefits}. For sequence modeling tasks in NLP, there are two stages of the previous research.

\paragraph{Theoretical Proposals.}
\citep{pestun2017tensor} propose a tensor network language model aims to construct the long-range correlation in real-world language modeling, \citep{pestun2017language} propose a quantum statistical language model on a one-dimensional lattice which is called trace-density model. 
To the best of our knowledge, these tensor network language models have remained a theoretical proposal instead of an empirical one.

\paragraph{Sequence Modeling.} \cite{novikov2021tensor} propose a new efficient tensor train-based approach to tensor-train density estimation that allows efficient computation of probability density function. \cite{miller2021tensor} apply a recurrent tensor network, uniform matrix product state, to the probabilistic sequence modeling while opening significant new research directions in the design of sequential generative models. However, probably due to the efficiency issue, most of the existing models have only been tested on small vocabulary size and simple datasets like Tomita grammars \citep{tomita1982dynamic}, further evaluation on moderately-scaled natural language datasets is necessary to fully assess their performance. We provide a comparison of the datasets we used with related work in the Appendix: \ref{sec:datasets}.

This paper is the first work to  derive a tensor network language model in a way that can be applied to real-world language modeling datasets. The efforts to improve efficiency are twofold. First, we do not calculate the probability normalization term used for the total probability law; instead, we turn to  calculate conditional probabilities based on context representations, as described in Sec. \ref{sec:recursive probability}.
Second, we further decompose TT cores and use  low-scale hidden units, see in Sec. \ref{sec: Tinyttlm}.

\input{sections/Preliminaries}

\section{Language Modeling Using Tensor Trains}
\label{sec:ttlm}

We introduce a language model in tensor space in Sec.\ \ref{sec:lmts}, and define our \emph{Tensor Train Language Model} in Sec.\ \ref{sec: TTLM}. 




\subsection{Language Models in a Tensor Space}
\label{sec:lmts}





Natural language typically has complex dependencies between features (e.g., tokens or words) \citep{hou2013mining}\footnote{Such  dependencies (including collocation) have been viewed as an analogy of entanglement \citep{hou2013mining}.} that are not captured well by standard methods such as feature concatenation.  One could also see a similar interaction between any arbitrary features in factorization machines \citep{rendle2010factorization}. Given text consists of $N$ words $X = [x^{(1)}, x^{(2)}, \cdots ,x^{(N)}]$ and a feature extractor $\vf_i \in \mathbb{R}^{I_i}$ (it can be one-hot encoding or word embedding), we now define a representation of $X$ designed to capture these dependencies:
\begin{equation} 
\begin{split}
    \label{eq: Phi X}
    \Phi(X) 
    &= \vf_1 (x^{(1)}) \otimes  \vf_2 (x^{(2)}) \cdots \otimes \vf_N (x^{(N)}) \\
    &= \bigotimes^N_{i=1} \vf_i (x^{(i)}) \\
\end{split}
\end{equation}
where the tensor space is  $\mathbb{R}^{I_1}\otimes\mathbb{R}^{I_2} \otimes\cdots\otimes\mathbb{R}^{I_N}$. Each component of $\vf_i$ represents independent meaning-bearing units, such as morphemes or latent factors.  For simplicity, we assume that a text shares the same one-hot encoding $\vf(x^{(t)}) \in \mathbb{R}^{|V|}$ in later sections. Consequently, $\Phi(X)$ is a $|V|^N$-dimensional tensor that records all possible combinations of words in $X$. 

Inspired by \cite{zhang2019generalized, kossaifi2020tensor}, we define a tensor regression model to compute the probability for each text $X$:
\begin{align} 
    \label{eq: general definition}
    p(X)
    &= \langle\mathcal{A}, \Phi(X)\rangle \nonumber \\
    &= \sum_{i_1,i_2,\cdots,i_N = 1}^{|V|} \mathcal{A}_{i_1,\cdots, i_N} \cdot \Phi(X)_{i_1,\cdots, i_N}
\end{align}
where $\langle \cdot \rangle$ denotes the inner product of two same-sized tensors, and  $\mathcal{A}$ is a regression weight tensor of the same shape as $\Phi(X)$ in the tensor space $\mathbb{V}^{\otimes N} = \underbrace{\mathbb{V} \otimes \cdots \otimes \mathbb{V}}_N $ where $\mathbb{V}$ refers to $\mathbb{R}^{|V|}$. Similar functions were considered in \cite{novikov2016exponential, NIPS2016_6211, khrulkov2018expressive, zhang2019generalized}.


\subsection{Tensor Train Language Model}
\label{sec: TTLM}
\subsubsection{Tensor-Train Decomposition}
Suppose the sequence of indices of words in the text $X$ is  $w_1, w_2, \cdots, w_N$, where  $w_i \in \{1, 2, \cdots, |V|\}$ and its corresponding weight in $\mathcal{A}$ is denoted as $\mathcal{A}_{w_1 w_2\cdots w_N}$. We use TT decomposition to  represent $\mathcal{A}_{w_1 w_2\cdots w_N}$ in the TT format \citep{oseledets2011tensor} as follows:

\begin{align} 
    \label{eq: TT_index}
    \mathcal{A}_{w_1w_2...w_N} 
    & = \underbrace{\tG^{(1)}_{:, w_1}}_{1\times R_1} \underbrace{\tG^{(2)}_{:, w_2, :}}_{R_1 \times R_2} \cdots \underbrace{\tG^{(N)}_{:, w_N}}_{R_{N-1} \times 1}  \\
    &= \sum_{\alpha_1, \cdots, \alpha_{N-1}}   \etG^{(1)}_{w_1\alpha_1} \etG^{(2)}_{\alpha_1w_2\alpha_2} \cdots \etG^{(N)}_{\alpha_{N-1}w_N}  \nonumber
\end{align}

where the tensors  $\tG^{(t)} \in \mathbb{R}^{R_{t-1} \times |V| \times R_t}$ ($t =1,...,d$, $R_0 = R_N = 1$ by definition) are called \textit{TT cores}, and  $R_k$ for $k = 1,\cdots,N$ are called \textit{TT ranks}. 

Despite its site-dependent TT cores $\tG^{(t)}$ potentially giving it more expressiveness for language modeling, this property currently generates unnecessary obstacles to its applicability, like the choice of $R_t$. Here we follow the convention of considering a special class of TT decompositions \cite{khrulkov2018expressive, miller2021tensor}, i.e. supposing all the intermediate TT cores are equal to each other $\tG = \tG^{(2)},\ldots,\tG^{(N-1)} \in \mathbb{R}^{R \times |V| \times R}$ and $\tG^{(1)} = \tG^{(N)} \in \mathbb{R}^{|V| \times R}$ in Eq. \ref{eq: TT_index}. 

\subsubsection{Definition of TTLM}
\begin{figure*}[t]
    \centering
    \includegraphics[scale=0.5]{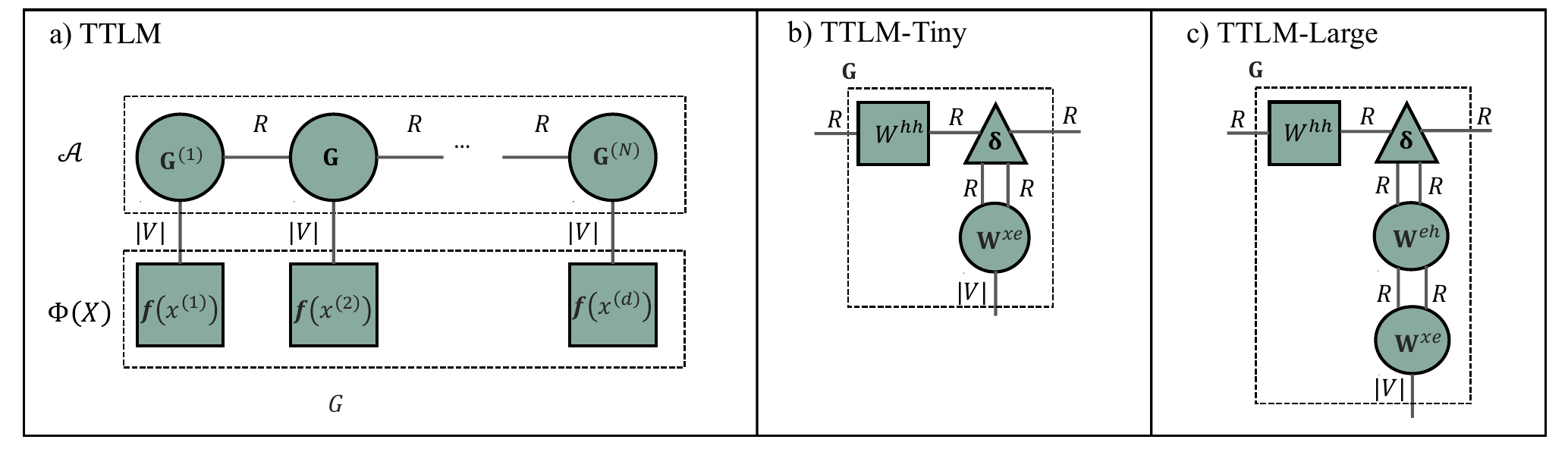}
    \caption{a) Tensor Train Language Model based on Eq. \ref{eq: ttlm_general}. b) TT core of TTLM-Tiny. c) TT core of TTLM-Large. The dashed line in the square represents $\mathcal{A}, \Phi(X)$, or $\tG$. Note that the only difference between TTLM-Large and TTLM-Tiny is whether to use tensor $\tW^{eh}$.}
    \label{fig: TTLM}
\end{figure*}

We define \textit{Tensor Train Language Model} (TTLM) as: 
\begin{equation}
\begin{split}
    p(X) = \sum_{i_1, \cdots, i_N = 1}^{|V|} \sum_{\alpha_1, \cdots, \alpha_{N-1} = 1}^R f(x^{(1)})_{i_1}\etG^{(1)}_{i_1\alpha_1} \cdots \\ f(x^{(N)})_{i_N}\etG^{(N)}_{\alpha_{N-1}i_N}  \label{eq: ttlm_general}
\end{split}
\end{equation}

where each $\vf(x^{(t)})$ is a one-hot vector having $w_t = 1$ for at most one $t$, and has zeros elsewhere. The tensor diagram notation of TTLM is shown in Fig. \ref{fig: TTLM}a. Note that  Eq. \ref{eq: ttlm_general} can compute the elements
of $\mathcal{A}$ in the low-dimensional space as Eq. \ref{eq: TT_index} does. This can be observed  if we represent the elements of  $\tG^{(t)}_{:, w_t, :}$ in Eq.\ \ref{eq: TT_index} as:
\begin{align}
    \label{eq: single G}
    \etG^{(t)}_{\alpha_{t-1}w_t\alpha_{t}}
    & = \sum_{{i = 1}}^{|V|} f(x^{(t)})_{i}\etG^{(t)}_{\alpha_{t-1}i\alpha_{t}} 
\end{align}
Since  $\mathcal{A}_{w_1w_2...w_N}$  here equals to  $p(X)$, we can derive Eq. \ref{eq: ttlm_general} by inserting Eq. \ref{eq: single G} into Eq.\ \ref{eq: TT_index}.

The critical difference between TTLM defined by Eq. \ref{eq: ttlm_general} and  Eq. \ref{eq: TT_index}  is that TTLM has combined the weights $\mathcal{A}$ and the input data $\Phi(X)$  together, indicating its potential to be used for language modeling tasks.

\subsubsection{Recursive Probability Computation.}
\label{sec:recursive probability}
We recursively unfold the calculation of TTLM in Eq. \ref{eq: ttlm_general} and find that $\tG$ has two sources of ``input'': the information from the previous recursive unfolding, and the input data $\vf(x^{(t)})$ (see Eq. \ref{eq: TTLM_unfold} for a detailed version). 
From this perspective, $\tG$ acts as a bilinear map $\tG: \mathbb{R}^{|V|} \times \mathbb{R}^R \rightarrow \mathbb{R}^R $, and we can regard the information in the previous step as a hidden state $\vh^{(t)}_\text{TTLM}$, given by:
\begin{align}
    \label{eq: ttlm cell full}
    \vh^{(t)}_{\textrm{TTLM}}
    & = \vf(x^{(t)})^T\tG  \vh^{(t-1)}_{\textrm{TTLM}} 
\end{align}
where $\vf(x^{(t)})$, $\tG$, and  $\vh^{(t-1)}_{\textrm{TTLM}}$ are contracted together (we permute the indices of $\tG$ from $\mathbb{R}^{R \times |V| \times R}$ to $\mathbb{R}^{|V| \times R \times R}$ which does not change the number of indices). 

Utilizing this recursive property, we here provide further details about computing $p(X)$ by TTLM in practice. In language modeling, $p(X)$ is often decomposed using the chain rule \citep{bahl1983maximum} as follows:
\begin{align}
    p(X) = \prod_{t=1}^N p(x^{(t)}|x^{(1:t-1)}) \nonumber
\end{align}
where $x^{(1:t-1)}$ denotes the text $[x^{(1)}, x^{(2)}, \cdots ,x^{(t-1)}]$.  At time $t$, the output prediction of a model, $\vy^{(t)}\in \mathbb{V}$, is a probability distribution of word $x^{(t)}$ given $x^{(1:t-1)}$. 

In TTLM, we define $\vy^{(t)}$ as follows: 
\begin{align}
\label{eq: probability_TT}
    \vy^{(t)} = \psi \left(\tG^{(t)} \vh^{(t-1)}_{\textrm{TTLM}}\right)
\end{align}
where $\tG^{(t)}\in \mathbb{R}^{|V| \times R}$ is the last TT core in TT format at time $t$. $\psi$ is any function that ensures that $\vy^{(t)}$ is non-negative and that the conditional probabilities sum to $1$. For the use of TTLM as a component in a larger architecture, $\psi$ can be chosen as a constant scaling function to preserve linearity; for stand-alone use of TTLMs, $\psi$ can be chosen to be any appropriate activation function---in the remainder of the paper, we shall use the $\softmax$ function \cite{bridle1990probabilistic}. Fig. \ref{fig: probability_ttlm} provides an example of a recursive calculation of conditional probability.
 \begin{figure}[t]
    \centering
\includegraphics[width=0.40\textwidth]{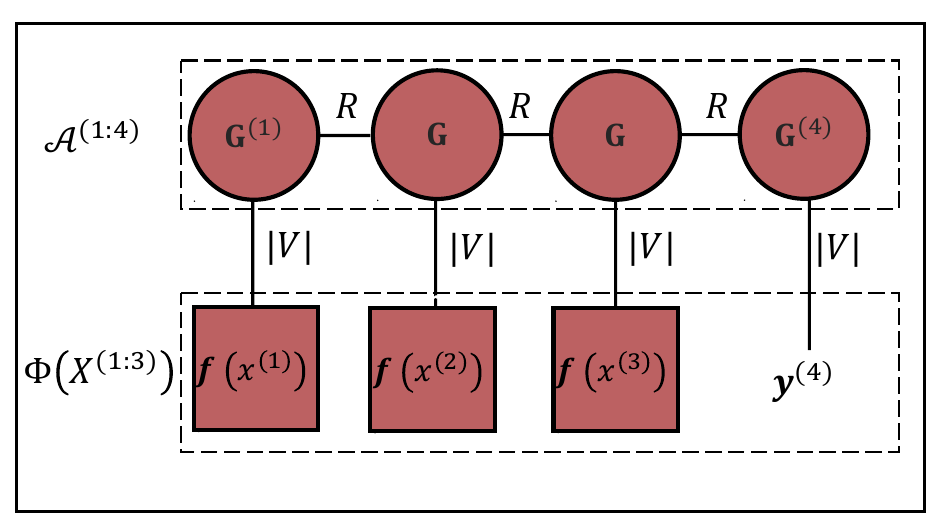}
    \caption{Recursive calculation of conditional probability in TTLM. Here we provide an example that given the text $x^{(1:3)}$, $\vy^{(4)}=\psi(\tG^{(4)}\vh_{\textrm{TTLM}}^{(3)})$ where $\vy^{(4)}$ is the probability distribution of word $x^{(4)}$. }
    \label{fig: probability_ttlm}
\end{figure}

We can derive  the definition of $\vy^{(t)}$ in high-dimensional space, if we substitute $\vh^{(t-1)}_{\textrm{TTLM}}$ in  Eq. \ref{eq: probability_TT}  by Eq. \ref{eq: ttlm_general} and Eq. \ref{eq: ttlm cell full}:
\begin{align} \small
      \vy^{(t)} &= \psi \left( \sum_{i_1, \cdots, i_{t-1}} \sum_{\alpha_1,\cdots, \alpha_{t-1}} f(x^{(1)})_{i_1}\etG^{(1)}_{i_1\alpha_1} \cdots \tG^{(t)}_{\alpha_{t-1}}  \right) \label{eq: probability_original_1} \\
    &=  \psi \left(\langle \mathcal{A}^{(1:t)}, \Phi(X^{(1: t-1)}) \rangle_{t-1} \right)  \label{eq: probability_original_3}
\end{align}
where $\mathcal{A}^{(1:t)}\in \mathbb{V}^{\otimes t}$,  $\Phi(X^{(1: t-1)}) = \bigotimes\limits^{t-1}_{i=1}  \vf(x^{(i)}) \in \mathbb{V}^{\otimes t-1}$ and $\langle \cdot \rangle_{t-1}$ denotes the "generalized inner product" defined in Sec. \ref{sec:preliminaries}. Note that Eq. \ref{eq: probability_original_1} is the low-dimensional form of Eq. \ref{eq: probability_original_3}, similarly to the relationship between Eq. \ref{eq: ttlm_general} and Eq. \ref{eq: general definition}.

By these definitions, there are some interesting properties of TTLM. (1) We can use \emph{teacher forcing} \citep{marcus1998rethinking} to learn parameters of TT cores. (2)  The hidden-to-output  tensor $\tG^{(t)}$ is defined to be the same as the input-to-hidden tensor $\tG^{(1)}$. (3) $\tG$ and $\tG^{(t)}$ have no parameters in common. We provide a detailed explanation of the  relationship between different TT cores in Appendix \ref{sec: relationship}.

\section{TTLM Variants}

To show the versatility and practical applicability of the TTLM framework, we now propose two new variants: TTLM-Large and TTLM-Tiny in Sec. \ref{sec: Tinyttlm}.  We briefly summarize the relationship between TTLM and some widely-used RNNs in Sec.\ \ref{sec: Existing  TTLM variants}.

\subsection{New Variants: TTLM-Large and TTLM-Tiny}
\label{sec: Tinyttlm}

The TT core $\tG$ in TTLM is an entire third-order tensor. In the two variants, we decompose $\tG$ into several separate tensors without violating the TT format, as shown in Fig.\ \ref{fig: TTLM}b and Fig.\ \ref{fig: TTLM}c. We define TTLM-Tiny and TTLM-Large as follows:

\begin{equation}
\begin{split}
& \vh^{(t)}_{\textrm{Tiny}} = \vf (x^{(t)})^T \tW^{xe} \bm{\delta} \mW^{hh}\vh^{(t-1)}_{\textrm{Tiny}}\\
& \vh^{(t)}_{\textrm{Large}}  = \vf (x^{(t)})^T \tW^{xe} \tW^{eh}  \bm{\delta} \mW^{hh}\vh^{(t-1)}_{\textrm{Large}}\\
\end{split}
\end{equation}

where $\mW^{hh} \in \mathbb{R}^{R \times R}$ is the hidden-to-hidden matrix;  $\tW^{xe} \in \mathbb{R}^{|V| \times R \times R}$ is the input-to-hidden tensor; $\tW^{eh} \in \mathbb{R}^{R \times R \times R \times R}$; and $\bm{\delta} \in  \mathbb{R}^{R \times R \times R \times R}$ is a fourth-order diagonal tensor such that $\delta_{ijkl} = 1$ if{f} the $i = j = k = l$, and $\delta_{ijkl} = 0$ otherwise.

The relationship between our proposed models and TTLM is as follows: $\tW^{xe}$ in both models take the same role as $\tG^{(t)}$ in TTLM (i.e. input-to-hidden and hidden-to-output), while  $\tG = \tW^{xe} \bm{\delta} \mW^{hh}$ in TTLM-Tiny and $\tG = \tW^{xe} \tW^{eh} \bm{\delta} \mW^{hh}$ in TTLM-Large.

As in RNNs, we compute the conditional probability recursively for TTLM-Large and TTLM-Tiny as:

\begin{equation}
\vy^{(t)} = \psi (\tV\tP\vh^{(t)})
\end{equation}

where $\tV\in\mathbb{R}^{R \times |V|\times R}$ is an output embedding tensor, $\tP\in\mathbb{R}^{R \times R\times R}$ is a projector tensor. Then we tie the input tensor $\tW^{xe}$ to the output embedding tensor $\tV$ (we provide a detailed explanation in Sec. \ref{sec:implementations}).

One obvious advantage of our models is to utilize information from the hidden layer and input data separately. Such interaction, particularly TTLM-Tiny, can potentially avoid overfitting, similarly to \cite{wu_multiplicative_2016} where \textit{multiplication integration} between two sources of "input" can outperform many other methods. In Sec \ref{sec: rank}, we provide relevant experimental evidence.

\subsection{Existing  TTLM Variants} 
\label{sec: Existing  TTLM variants}
Given the fact that TT scores of TTLM can vary,  Appendix \ref{sec: TTLM generalization} provides a detailed illustration that three existing models, namely Second-order RNNs,  Recurrent Arithmetic Circuits (RACs), and Multiplicative Integration RNNs (MI-RNNs)  can be considered as one of the ”special” implementations of TTLM. 

We briefly summarize the differences between the three models: 1) Second-order RNNs use the third-order $\tT$ as the TT cores with an activation function given Eq. \ref{eq: second-order-rnn}; 2)
RACs use $\mW^{hx} \odot \mW^{hh}$ as the TT cores given Eq.\ \ref{eq: RAC_simple};  3) MI-RNNs use $\mW^{hx} \odot \mW^{hh}$  as the TT cores with an activation function given Eq. \ref{eq: mirnn}. 

Along with our two proposed models, we study the experimental performance of second-order RNNs, RACs and MI-RNNs compared to TTLM-Large and TTLM-Tiny in Sec. \ref{sec: experiments}.

\section{Experimental Evaluation} 
\label{sec: experiments}

To further understand the properties of TTLM variants, we now investigate the effectiveness of TTLM, TTLM-Large and TTLM-Tiny compared to Second-order RNNs, RACs, MI-RNNs, and Vanilla-RNNs.

We specify our  experimental setting in Sec. \ref{sec:setting} and implementation details in Sec. \ref{sec:implementations}. We study the influence of ranks on the performance of TTLM variants in Sec \ref{sec: rank} and examine the impact of nonlinear activation functions on the effectiveness of TTLM variants in Sec.\ref{sec: activation}.

\subsection{Experimental Setting}
\label{sec:setting}
\begin{table*}[t]
\small
\centering
  \begin{tabular}{l|cc|cc|ccc}
    \toprule
      \multirow{2}{*}{Model} & \multicolumn{2}{c|}{
      WikiText-2} &
    \multicolumn{2}{c|}{PTB} & Hidden & Layer & Embed  \\
    
    &Param & PPL  & Param  & PPL & (Rank) & & Size\\
    
    \midrule
    Transformer \citep{vaswani2017attention} & 90.5M & 293.0 & 32.8M & 208.7 &20 & 1 & 400 \\
    Vanilla-RNNs \citep{mikolov2012context} & 11.6M & 96.6 & 4.0M & 115.3 & 20 & 1 & 400\\
    Second-order RNNs \citep{hochreiter1997long} & 11.8M & 96.0 & 4.2M& 108.2 & 20 & 1 & 400\\
    RACs  \citep{levine2018benefits} &11.6M& 97.6 & 4.0M  & 116.8 &20 & 1 & 400 \\
    MI-RNNs \citep{wu_multiplicative_2016} &  11.6M & 99.6 &4.0M & 119.1 & 20 & 1 & 400 \\
    \hline
    TTLM & 12.2M & 546.4 & 4.2M & 559.8 & 20& 1 & 400\\
    TTLM-Tiny & 11.6M &94.9 & 4.0M &106.8 & 20 & 1 & 400\\
    TTLM-Large &11.8M &\textbf{82.3} & 4.2M & \textbf{99.3} & 20 & 1 & 400 \\
  \bottomrule
\end{tabular}
\caption{Test set PPL on the WikiText-2 and PTB datasets. The symbol "$-$" means these data are not available in their original paper. The “Param” column denotes the number of parameters; see Sec. \ref{sec:implementations} for a detailed description.  The "Hidden (Rank)" column  denotes the number of hidden units or ranks. The "Embed Size" column denotes the size of each embedding vector.  We report the lowest test set PPL of the Transformer whose number of heads is  selected from [2, 4, 5, 8]. \label{tab:evaluation}}
\end{table*}
\paragraph{Task, Datasets, and Metric.}  We conduct experiments on two word-level language model datasets: \textbf{(1)} English Penn Treebank (PTB) \citep{marcinkiewicz1994building}, which consists of 929k training tokens, 73k validation tokens, and 82k test tokens. Its vocabulary size is 10k. \textbf{(2)} The WikiText-2 dataset \citep{merity2016pointer} is derived from Wikipedia articles and consists of 2088k training tokens, 217k validation tokens, 45k test tokens, and a vocabulary of over 30k types. We compare these models on the language modeling task,  evaluated by the Perplexity (PPL) \citep{meister-cotterell-2021-language}; the lower the PPL, the better the model.


\paragraph{Baselines.} Our models are compared with the following baselines: 
  Transformer \citep{vaswani2017attention} ,  Vanilla-RNNs \citep{mikolov2012context},     Second-order RNNs \citep{hochreiter1997long}, Recurrent Arithmetic Circuits (RACs)   \citep{levine2018benefits}, and   Multiplicative Integration RNNs (MI-RNNs) \citep{wu_multiplicative_2016}. The implementation details are provided in Sec. \ref{sec:implementations}.









\paragraph{Hyperparameters.} \textbf{(1)} To compare the effectiveness of comparable models on the same scale, we set the rank/hidden units of TTLM variants/Vanilla-RNNs as [5, 10, 20, 25, 30, 35, 40, 45, 50]. The embedding size of these models is the squared number of hidden units/ranks. This setup is because of the architectures of TTLM-Large and TTLM-Tiny as introduced in Sec. \ref{sec: Tinyttlm}. \textbf{(2)}  To avoid the potential impact of the embedding size on Vanilla-RNNs' performance, we provide several common choices of embedding size in the model by setting its embedding size as [100, 200, 300]. We name them as RNNs-100, RNNs-200, and RNNs-300 correspondingly and display them in Fig. \ref{fig:rank}. \textbf{(3)}  We train all models for 50 epochs and choose the best model in the validation set to predict the result in the test set. \textbf{(4)}   The weights in the models are adjusted to minimize the average cross entropy loss over training sequences via stochastic gradient descent computed using the truncated backpropagation through time algorithm \citep{werbos1990backpropagation,williams1990efficient}.  The random seed is fixed to ensure the experimental results are not influenced by initializing the weights.

\subsection{Implementations}
\label{sec:implementations}

\begin{table}[ht] \small
    \centering
    \begin{tabular}{ll}
    \toprule
        Model & Training Parameters \\
        \midrule
        Vanilla-RNN & $\mW^{xe}\in \mathbb{R}^{E\times |V|}$, $\mW^{eh}\in \mathbb{R}^{E\times H}$, \\ 
        & $\mW^{hh}\in \mathbb{R}^{H\times H}$, $\mP\in\mathbb{R}^{H\times E}$, \\
        & $\mV\in\mathbb{R}^{E\times|V|}$  \\
        \hline
        MI-RNNs & $\mW^{xe}\in \mathbb{R}^{E\times |V|}$, $\mW^{eh}\in \mathbb{R}^{E\times H}$, \\ 
        & $\mW^{hh}\in \mathbb{R}^{H\times H}$, $\mP\in\mathbb{R}^{H\times E}$, \\
        & $\mV\in\mathbb{R}^{E\times|V|}$  \\
       
        \hline
        RACs & $\mW^{xe}\in \mathbb{R}^{E\times |V|}$, $\mW^{eh}\in \mathbb{R}^{E\times H}$, \\
        & $\mW^{hh}\in \mathbb{R}^{H\times H}$, $\mP\in\mathbb{R}^{H\times E}$, \\
        & $\mV\in\mathbb{R}^{E\times|V|}$\\
        \hline
        Second-order RNNs & $\mW^{xe} \in \mathbb{R}^{E\times |V|}$, $\tT \in \mathbb{R}^{H\times H \times H}$, \\         
        & $\mW^{hh}\in \mathbb{R}^{E \times H}$, $\mP\in\mathbb{R}^{H\times E}$, \\
        & $\mV\in\mathbb{R}^{E\times|V|}$ \\
        \hline
        TTLM & $\tG\in \mathbb{R}^{R\times |V|\times R}$, $\mG^{(t)}\in \mathbb{R}^{R \times |V|}$, \\ 
        
        & $\mG^{(1)} \in \mathbb{R}^{R \times |V|}$ \\
        \hline
        TTLM-Tiny & $\tW^{xe}\in \mathbb{R}^{R\times |V| \times R}$, $\mW^{hh}\in\mathbb{R}^{R\times R}$, \\
        & $\tP\in\mathbb{R}^{R\times R\times R}$, $\tV\in\mathbb{R}^{R\times R\times |V|}$\\
        
        \hline
        TTLM-Large & $\tW^{xe}\in\mathbb{R}^{R\times |V|\times R}$, \\ & $\tW^{eh}\in\mathbb{R}^{R\times R\times R\times R}$, \\
        & $\mW^{hh}\in\mathbb{R}^{R\times R}$, $\tP\in\mathbb{R}^{R\times R\times R}$, \\
        &  $\tV\in\mathbb{R}^{R\times R\times |V|}$\\
    \bottomrule
    \end{tabular}
    \caption{Training parameters in our implementation. $E$ is the embedding size, $H$ is the hidden units in RNNs, and $R$ is the rank in the TTLM.  We set $H=R$ and $E=R^2$ to make the parameters of all models in the same scale. The parameters of $\mW^{xe}$ are uniformly initialized in the interval $[-0.1,0.1]$, $\mW^{eh}$ and $\mW^{hh}$ are uniformly initialized between $[-\frac{1}{\sqrt{H}}, \frac{1}{\sqrt{H}}]$.}
    \label{tab:my_label}
\end{table}

We implement all models using PyTorch on GPU A100 with one single card. 

We use the PyTorch version of the standard Transformer \citep{vaswani2017attention}.\footnote{The code is available at \url{https://pytorch.org/tutorials/beginner/transformer_tutorial.html}.} For RNNs, there are five matrix parameters: $\mW^{xe}\in \mathbb{R}^{E\times |V|}$ is the input embedding matrix, $\mW^{eh}\in
\mathbb{R}^{E\times H}$ is the embedding-to-hidden matrix, $\mW^{hh}\in \mathbb{R}^{H\times H}$ is the hidden-to-hidden matrix. We tie (share the same training parameters) the input embedding  $\mW^{xe}$ and output embedding $\mV$, which has been proved to lead to a significant reduction in perplexity \citep{press2016using}. So there is a projection matrix $\mP\in\mathbb{R}^{H\times E}$ before the output embedding. All this process is introduced in  \citep{press2016using}. 

For TTLM models, we tie the input tensor $\tW^{xe}$ and $\tV$. The implementation of $\delta$ is functioned by a reshape function, so the interaction between hidden and input can be computed by matrix product. We also let $\tG^{(1)}$ have the same parameters along the dimension $|V|$ (i.e., $\tG^{(1)}$ is simplified into a  $\mG^{(1)}\in\mathbb{R}^{1\times R}$ and thus it can be viewed as the initial hidden state).

\subsection{Rank and Effectiveness Analysis}
\label{sec: rank}

\begin{figure}[t]
    \centering
    \includegraphics[scale=0.36]{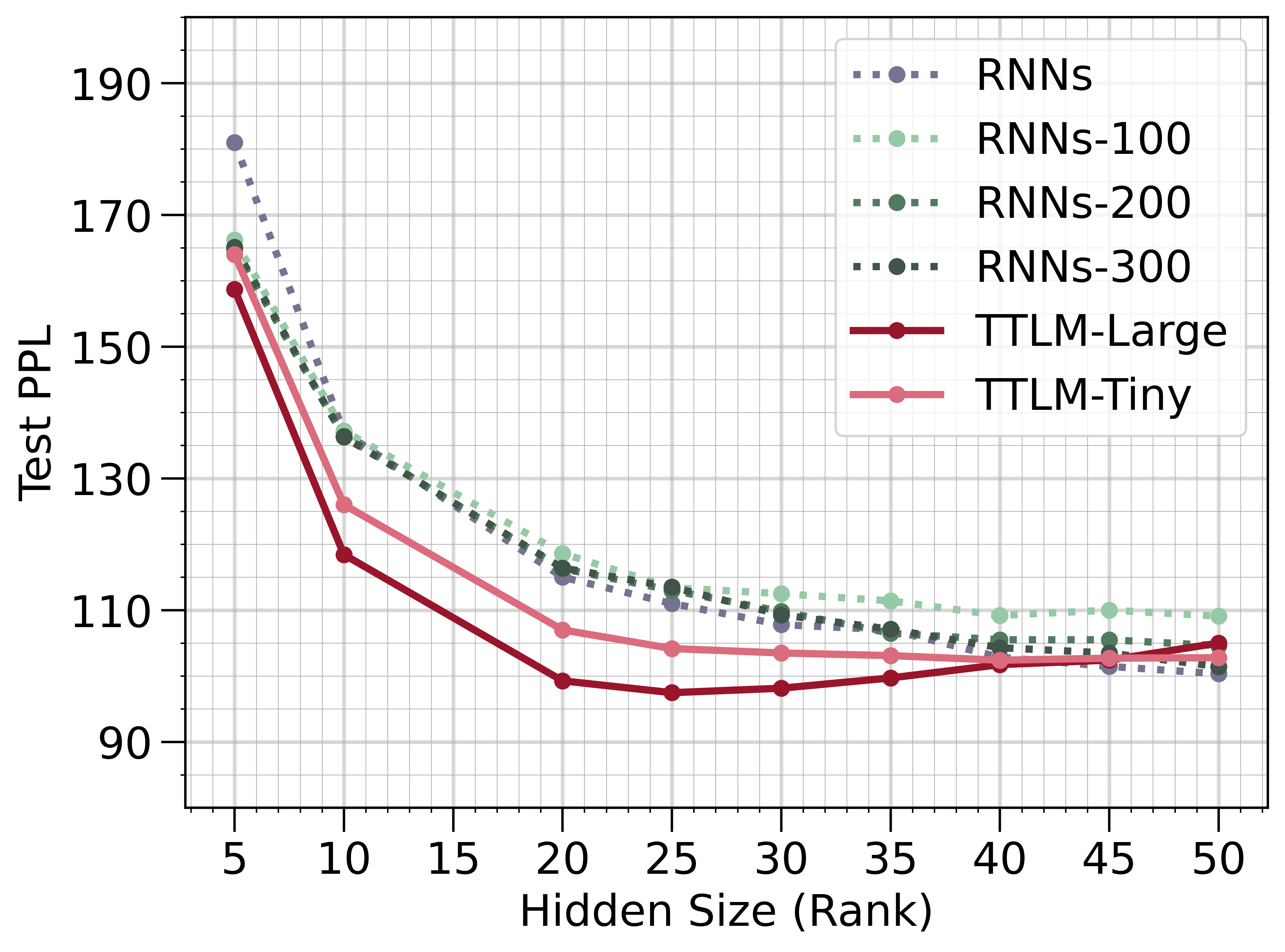}
    \caption{Test set PPL  on the PTB dataset w.r.t. ranks/hidden units. RNNs here denotes Vanilla-RNNs, which has the same embedding size as TTLM-Large and TTLM-Tiny. RNNs-100, RNNs-200, and RNNs-300 are the Vanilla-RNNs with fixed embedding sizes of 100, 200, and 300, respectively.}
    \label{fig:rank}
\end{figure}

The rank of the TT format has been used to explain the expressive power or long-term memory capacity of RNNs \citep{khrulkov2018expressive, levine2018benefits}. The \textit{rank} of TT decomposition has been proved to be the dimension of the hidden states of RNNs \citep{khrulkov2018expressive}, which reflect on the capacity of the RNNs. The higher rank can have more capacity and vice versa.  However, the relationship between rank and effectiveness in language modeling has yet to be shown practically. We will evaluate the effectiveness of TTLM-Large and TTLM-Tiny w.r.t ranks.

\subsubsection{Effectiveness}  
Table \ref{tab:evaluation} presents the results of the test set PPL for our models and the baselines on the WikiText-2 and PTB datasets. As shown, compared to Vanilla-RNNs, TTLM-Large reduces PPL by 14.3 and 16.0, respectively, and TTLM-Tiny reduces PPL by 1.7 and 8.5, respectively. Thus, when the number of hidden units or ranks is set to 20 and the embedding size is 400, both TTLM-Large and TTLM-Tiny perform better than all the baselines.
\begin{figure}[t]
    \centering
    \includegraphics[scale=0.24]{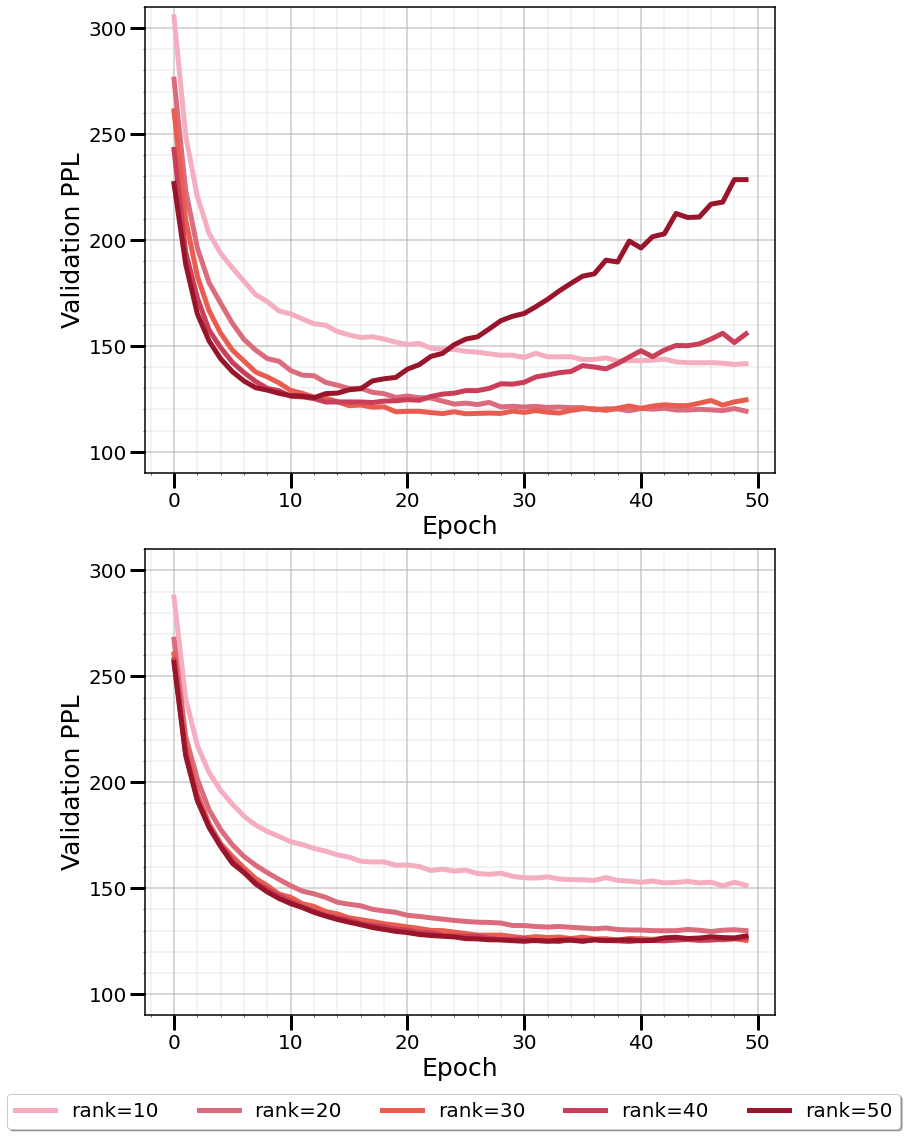}
    \caption{Validation set PPL of TTLM-Large and TTLM-Tiny with increasing ranks on the PTB dataset. Top: TTLM-Large, Bottom: TTLM-Tiny.}
    \label{fig:rank_analysis}
\end{figure}

To further evaluate the effectiveness of our models, we conduct a comparison between TTLM-Large, TTLM-Tiny, and Vanilla-RNNs using increasing ranks, as depicted in Fig. \ref{fig:rank}. It's worth noting that we also include RNNs-100, RNNs-200, and RNNs-300 to control for the potential impact of large embedding size.  As shown, even when the number of hidden units reaches 40, the test set PPL of RNNs decreases steadily, while TTLM-Large and TTLM-Tiny do not.  Thus, we expect our models to outperform Vanilla-RNNs with a \textit{low-scale} of hidden units (i.e., the number ranges from 5 to 40), but not larger scales. 

\subsubsection{Overfitting}  
\begin{figure*}[t]
    \centering
    \includegraphics[width= 1 \textwidth]{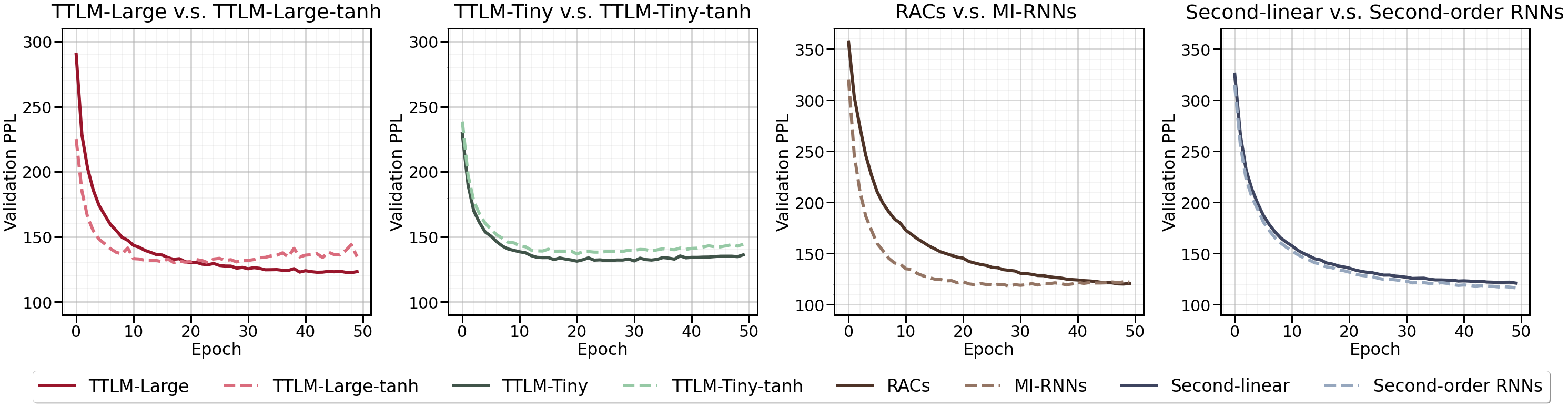}
    \caption{The influence of nonlinearity on TTLM variants on the PTB dataset. The suffix \texttt{-tanh} refers to a model using the $tanh$ activation function, indicated by dashed lines. Second-linear refers to Second-order RNNs without activation functions. Setting: all models have 17 hidden units/ranks.}
    \label{fig:activation_comparision}
\end{figure*}
Fig.\ \ref{fig:rank_analysis}  illustrates the performance of TTLM-Large and TTLM-Tiny on the validation set as the number of ranks
 increases. As shown,  the validation set PPL of TTLM-Large starts to rise in earlier training epochs when we gradually enlarge its ranks. In contrast, the validation PPL of TTLM-Tiny stably decreases as the number of ranks increases. The  comparison indicates that TTLM-Large is more prone to overfitting than TTLM-Tiny. This finding is further supported by the results in Fig. \ref{fig:rank}. The test set PPL of TTLM-Tiny consistently improves as the rank increases, while TTLM-Large's performance declines when the number of ranks reaches 25.

When we focus on the difference between the two models, TTLM-Large has an additional parameter tensor $\tW^{eh}$. Thus, we believe that the simpler parameterization of the TT cores, the more easily the model avoids overfitting. This finding is consistent with the comparison between MI-RNNs and Second-order RNNs by  \cite{wu_multiplicative_2016}. When it comes to practical  situations, we need to be aware that TTLM-Tiny has a lower capacity to fit the training data and, as a result, poses a lower risk of overfitting when compared to TTLM-Large.

\subsection{Nonlinearity  Analysis}
\label{sec: activation}
Previous studies have attempted to use TT decomposition as a theoretical platform to investigate RNNs \citep{khrulkov2018expressive, levine2018benefits}. However, one key difference between TT decomposition and the existing neural networks (like Vanilla-RNNs) is the nonlinearity activation functions inside the network models. The lack of nonlinearity in the tensor decomposition calls into question whether its theoretical analysis is transferable to models based on RNNs. To understand whether the effect of the $tanh$ activation function on the TT variants varies with the TT cores,  we provide an empirical result as displayed in Fig.\ \ref{fig:activation_comparision}.

Regarding convergence speed, $\tanh$ speeds up TTLM-Large-tanh, TTLM-Tiny-tanh, and MI-RNNs while barely influencing second-order RNNs. Regarding the magnitude of the lowest validation perplexity,  $\tanh$ impairs the performance  TTLM-Large and TTLM-Tiny but has little influence on multiplicative integration and the third-order tensor $\tT$ in Second-order RNNs.  

Thus, the influence of nonlinear activation functions on TTLM variants depends on TT cores settings, both for the convergence of validation PPL and the magnitude of the lowest validation PPL.  From an experimental point of view, we believe that the effect of nonlinearity functions on one TT variant cannot simply be transferred or analogized to another TT variant. This also suggests that one should be wary of the analogy between tensor decomposition and existing neural network models at the implementation level declared by previous research \citep{khrulkov2018expressive, levine2018benefits}. The nonlinear activation functions could be a factor influencing such an analogy.

\section{Conclusion}

Tensor networks having been proposed as promising language models, we first apply TT decomposition to real-world language modeling datasets and name the framework TTLM. We propose two variants: TTLM-Large and TTLM-Tiny, and show that they outperform Vanilla-RNNs with low-scale hidden units. The presentation of the experimental results is an advancement for exploring tensor networks in machine learning. Meanwhile, we demonstrate that Second-order RNNs, RACs, and MI-RNNs are special implementations of TTLM. 

A limitation of this study is that it shall examine the influence of different normalization functions. In future research, if appropriate mathematical tools and benchmarks are available, we plan to investigate the long-range correlation modeling capability of TTLM in natural language, which is believed to be one of the core features in TTLM. 


\bibliography{icml2023}
\bibliographystyle{icml2023}

\newpage
\appendix
\onecolumn

\section{Relationship between TT Cores in TTLM}
\label{sec: relationship}

To help readers understand the roles of TT cores in TTLM, we here provide a detailed calculation of the probability of a text $X = [x^{(1)}, x^{(2)}, \cdots ,x^{(N)}]$ by TTLM. Note that all the intermediate TT cores are equal to each other: $\tG = \tG^{(2)},...,\tG^{(N-1)}$ and $\tG^{(1)} = \tG^{(N)}$. 

The calculation of $\vy^{(t)}$ (i.e. the conditional probability of $x^{(t)}$ given  $x^{(1:t-1)})$ at time $t$) can be described as three steps. As step I, suppose $\vf(x^{(1)})$ is a one-hot vector having $f(x^{(1)})_1 = 1$. The calculation of $\tG^{(1)}\vf(x^{(1})$ in TTLM is as follows:

\begin{align}
\tG^{(1)}\vf(x^{(1}) &= 
\left[\begin{array}{c}
f\left(x^{(1)}\right)_1 \\
f\left(x^{(1)}\right)_2 \\
\ldots \\
f\left(x^{(1)}\right)_{|V|}
\end{array}\right] \left[\begin{array}{cccc}
\tG_{11}^{(1)} & \tG_{12}^{(1)} & \ldots & \tG_{1 R}^{(1)} \\
\tG_{21}^{(1)} & \tG_{22}^{(1)} & \ldots & \tG_{2 R}^{(1)} \\
\ldots & \ldots & \ldots & \ldots \\
\tG_{|V| 1}^{(1)} & \tG_{|V| 2}^{(1)} & \ldots & \tG_{|V| R}^{(1)}
\end{array}\right] \nonumber \\ 
&=\left[\tG_{11}^{(1)}, \tG_{12}^{(1)}, \cdots, \tG_{1R}^{(1)}\right]^T \nonumber \\ 
&=\vh_{\textrm{TTLM}}^{(1)} \nonumber
\end{align}


As step II, TTLM will calculate 
$\vf(x^{(i)})\tG\vh^{(i-1)}_{\textrm{TTLM}}$ where $i \in \{2,3,\cdots,t-1\}$. For example, $\vh_{\textrm{TTLM}}^{(2)}$ is calculated in Eq. \ref{eq: ttlm cell full} at time $t= 2$ as follows: 
\begin{align}
    \vh_{\textrm{TTLM}}^{(2)} = \vf(x^{(2)})^T\tG  \vh^{(1)}_{\textrm{TTLM}} \nonumber 
\end{align}

As step III, TTLM will output $\vy^{(t)}$ as follows:

\begin{align}
\tG^{(t)}\vh^{(t-1)}_{\textrm{TTLM}} &= 
\left[\begin{array}{cccc}
\tG_{11}^{(t)} & \tG_{12}^{(t)} & \ldots & \tG_{1 R}^{(t)} \\
\tG_{21}^{(t)} & \tG_{22}^{(t)} & \ldots & \tG_{2 R}^{(t)} \\
\ldots & \ldots & \ldots & \ldots \\
\tG_{|V| 1}^{(t)} & \tG_{|V| 2}^{(t)} & \ldots & \tG_{|V| R}^{(t)}
\end{array}\right] \left[\begin{array}{c}
h^{(t-1)}_{{\textrm{TTLM}}_1} \\
h^{(t-1)}_{{\textrm{TTLM}}_2} \\
\ldots \\
h^{(t-1)}_{{\textrm{TTLM}}_R}
\end{array}\right] \nonumber \\ 
&= \left[\begin{array}{c}
\sum_{i=1}^{R}\tG_{1i}^{(t)} h^{(t-1)}_{{\textrm{TTLM}}_1}  \\
\sum_{i=1}^{R}\tG_{2i}^{(t)} h^{(t-1)}_{{\textrm{TTLM}}_2} \\
\ldots \\
\sum_{i=1}^{R}\tG_{Ri}^{(t)} h^{(t-1)}_{{\textrm{TTLM}}_R} 
\end{array}\right] \nonumber
\end{align}

Observing the calculation, $\tG^{(1)}$, $\tG$ and $\tG^{(t)}$ theoretically have no parameters in common (though we set $\tG^{(1)} = \tG^{(t)}$ for simplicity). Further, their roles in TTLM are different: $\tG^{(1)}$ can be viewed as a word embedding matrix; $\tG$ deals with two sources of information, i.e. hidden state and input word; $\tG^{(t)}$ extracts the evidence provided in $\vh^{(t-1)}_{\textrm{TTLM}}$ and generates a set of scores over vocabulary.

\section{Relationship between TTLM and some RNNs}
\label{sec: TTLM generalization}

We now demonstrate the relationship between TTLM and Second-order RNNs,  Recurrent Arithmetic Circuits (RACs) and Multiplicative Integration RNNs (MI-RNNs).  

To avoid symbol clutter when representing  different RNNs, the notation is:  $\mW^{hx} \in \mathbb{R}^{R \times |V|}$ denotes the input-to-hidden  matrix, $\mW^{hh} \in \mathbb{R}^{R \times R}$ denotes hidden-to-hidden  matrix, $\phi(\cdot)$ is an element-wise nonlinear activation function. Also, different hidden states are denoted as: Second-order RNNs ($\vh^{(t)}_\textrm{2nd}$), RACs ($\vh^{(t)}_\textrm{RAC}$) and MI-RNNs ($\vh^{(t)}_\textrm{MI}$).

\subsection{Relation to Second-order RNNs}
\label{appendix: seecond}

Unlike Vanilla-RNNs \citep{mikolov2012context} that have \textit{additive} blocks, Second-order RNNs have interaction between hidden states and input data in \textit{multiplicative} form.   This is achieved by a third-order tensor $\tT$ with the $i$-th coordinate of the hidden states $\vh^{(t)}_\textrm{2nd}$ defined as \citep{hochreiter1997long, maupome-meurs-2020-language}:
\begin{equation} \small
\label{eq: second-order-rnn}
h^{(t)}_{\textrm{2nd}_i}  = \phi(\vf(x^{(t)})^T \tT_{i, :, :}\vh^{(t-1)}_\textrm{2nd} + \vb)
\end{equation}

where $\tT_{i, :, :} \in\mathbb{R}^{M \times R}$ is the $i$th slice of tensor $\tT\in\mathbb{R}^{M\times R \times R}$, and $\vb$ is a bias vector. For simplicity, we will ignore  $\vb$ for other variants of RNNs since $\vb$  can be seen as $0$th component of $\vf (x^{(t)})$ which equals to 1. \cite{rabusseau2019connecting} has provided that Tensor Trains can  generalize linear Second-order RNNs. We here provide a basic proof from the perspective of recursive property in TTLM.

\begin{claim}
\label{cm: G and T}
The third-order tensor $\tT$ in Second-order RNNs equals the TT cores in TTLM. There is a nonlinear activation $\phi$ such that the hidden states of Second-order RNNs is identical to that of TTLM when they are accompanied by $\phi$.
\end{claim}

\begin{proof} The proof is based on the following observation: We recursively unfold the calculation of TTLM in Eq. \ref{eq: ttlm_general}: 

\begin{equation} \small
\begin{aligned}
\label{eq: TTLM_unfold}
p(X) &= \sum_{{i_ = 1}}^{|V|} f(x^{(1)})_{i_1}\etG^{(1)}_{i_1\alpha_{1}}  \cdots \\
&= \sum_{{i_1,i_2 = 1}}^{|V|} \sum_{\alpha_1 = 1}^R   f(x^{(1)})_{i_1}\etG^{(1)}_{i_1\alpha_{1}} f(x^{(2)})_{i_2} \etG_{\alpha_{1}i_2\alpha_{2}} \cdots  \\
&  \quad\quad\vdots \\
&= \sum_{i_1, \cdots, i_N = 1}^{|V|} 
    \sum_{\alpha_1, \cdots ,\alpha_{N-1} =1}^{R}  f(x^{(1)})_{i_1}\etG^{(1)}_{i_1\alpha_1}
    f(x^{(2)})_{i_2}\etG_{\alpha_1i_2\alpha_2} \cdots f(x^{(N)})_{i_N}\etG^{(N)}_{\alpha_{N-1}i_N}   \\
\end{aligned}
\end{equation}
Observe in the above, that at each time step,
$\tG$ has two sources of ``input'': the information from the previous recursive unfolding (\textit{e.g.}, in the second line, the first line is the previous information), and the input data $\vf(x^{(t)})$. 
From this perspective, $\tG$ acts as a bilinear map $\tG: \mathbb{R}^{|V|} \times \mathbb{R}^R \rightarrow \mathbb{R}^R $, and we can regard the information in the previous line as a hidden state $\vh^{(t)}_\text{TTLM}$, given by: 
\begin{equation} \small
\begin{split}
    \label{eq: ttlm cell}
    h^{(t)}_{\textrm{TTLM}_{\alpha_{t}}}
    & = \sum_{{i_t = 1}}^{|V|} \sum_{{\alpha_t = 1}}^{R} f(x^{(t)})_{i_t}\etG_{i_t\alpha_{t}\alpha_{t-1}}  h^{(t-1)}_{\textrm{TTLM}_{\alpha_{t-1}}} \\
\end{split}
\end{equation}
where we permute the indices  of  $\etG_{\alpha_{t-1}i_t\alpha_{t}}$ as $\etG_{i_t\alpha_{t}\alpha_{t-1}}$  ( note that this does not change the number of indices). 

We can also represent the hidden states in Second-order RNNs shown by Eq. \ref{eq: second-order-rnn} in element-wise fashion:
\begin{equation} \small
\begin{split}
    \label{eq: second_element}
    h^{(t)}_{\textrm{2nd}_i}
    &= \phi(\vf(x^{(t)})^T \tT_{i, :, :}\vh^{(t-1)}_\textrm{2nd}) \\
    &= \phi\left(\sum^{|V|}_{j=1}\sum^{R}_{k=1} f(x^{(t)})_j\etT_{jik} h^{(t-1)}_{\textrm{2nd}_{k}}\right) \\
\end{split} 
\end{equation}
where $j,k$ are the dummy indices as $i_t, \alpha_{t}$;  $i$ specifies the coordinate of $\vh^{(t)}_{2nd}$ just like  $\alpha_t$ for $\vh^{(t)}_\text{TTLM}$. Thus,  $\tT$ and $\tG$ are the same-sized trainable bi-linear map. 

After demonstrating that the third-order tensor $\tT$ in Second-order RNNs equals the TT cores $\tG$, the only difference between the hidden states in Eq. \ref{eq: second_element}  and in Eq. \ref{eq: ttlm cell} is $\phi$. If we add $\phi$ for $\vh^{(t)}_\text{TTLM}$, the hidden states of Second-order RNNs and TTLM are identical,  as shown in Fig. \ref{fig:rnns}a.

\end{proof}

\subsection{Relation to RACs and  MI-RNNs}
\label{appendix: sec: claim 4.3}

We here focus on Multiplicative Integration (MI), a way to connect two sources of inputs by the Hadamard product `$\odot$'. MI has been used in RACs, Multiplicative RNNs (M-RNNs) \citep{sutskever2011generating} and MI-RNNs: 

\emph{Recurrent Arithmetic Circuits} (RACs) are recurrent networks with  hidden states $\vh^{(t)}_\textrm{RAC}$ defined as  \citep{levine2018benefits}:
\begin{equation} \small
\label{eq: RAC_simple}
   \vh^{(t)}_\textrm{RAC} =  \mW^{hx} \vf (x^{(t)}) \odot \mW^{hh}\vh^{(t-1)}_\textrm{RAC}
\end{equation}

where these hidden states are also used as an \emph{intermediate term} in M-RNNs. 

\emph{Multiplicative Integration RNNs} (MI-RNNs) are RACs with an activation function and hidden states $\vh^{(t)}_\textrm{MI}$ defined as  \citep{wu_multiplicative_2016}:

\begin{equation} \small
\label{eq: mirnn}
    \vh^{(t)}_\textrm{MI}= \phi(\mW^{hx} \vf(x^{(t)}) \odot \mW^{hh}\vh^{(t-1)}_\textrm{MI})
\end{equation}



\begin{figure}[t]
    \centering
    \includegraphics[scale=0.45]{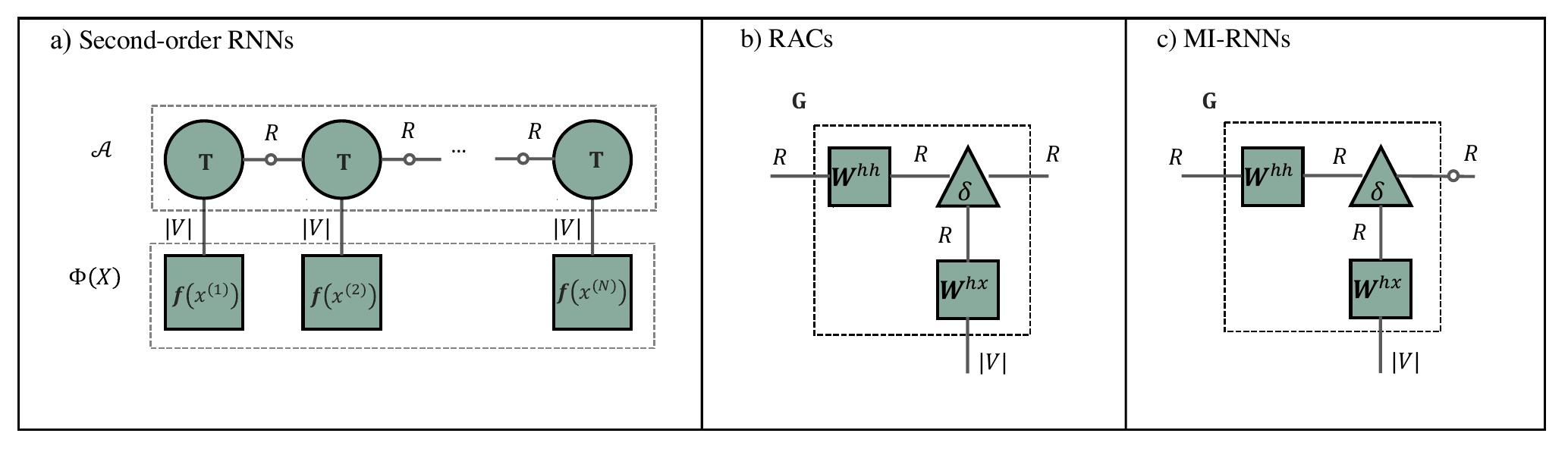}
    \caption{a) Second-order RNNs under TTLM framework.  b) Hidden state of RACs under TTLM framework. c) hidden state of MI-RNNs under TTLM framework. The dashed line in the square denotes $\mathcal{A}, \Phi(X)$ or $\tG$. The small hollow circles denote the activation functions.}
  \label{fig:rnns}   
\end{figure}




\begin{claim}
\label{cm: G and RACs} Given the condition the TT-scores: $\tG = \mW^{hx} \odot \mW^{hh}$. 
The hidden states of RACs are identical to that of TTLM. There is a nonlinear function $\phi$ such that the hidden states of MI-RNNs are identical to that of TTLM if they are accompanied by $\phi$.   
\end{claim}


\begin{proof} The proof is based on the following observation:
In the language of tensor contractions, Eq.\ \ref{eq: RAC_simple} involves contracting the input weights matrix  $\mW^{hx}$ with the input vector $\vf (x^{(t)})$, and contracting the hidden weights matrix $\mW^{hh}$ with $\vh^{(t-1)}_\textrm{RAC}$. The Hadamard product of the two is a third-order diagonal tensor $\delta \in \mathbb{R}^{R \times R \times R}$ such that $\delta_{ijk} = 1$ if{f} the $i = j = k$, and $\delta_{ijk} = 0$ otherwise. Thus, we can write Eq.\ \ref{eq: RAC_simple} in element-wise fashion: 
\begin{equation} \small
\begin{split}
    \label{eq: RAC_element}
    h^{(t)}_{\textrm{RAC}_{\alpha_t}}
    & = \sum_{{i_t = 1}}^{|V|} \sum_{{\alpha_t = 1}}^{R} f(x^{(t)})_{i_t}W^{hx}_{i_tj}\mathbf{\delta}_{j\alpha_{t}k}W^{hh}_{k\alpha_{t-1}} h^{(t-1)}_{\textrm{RAC}_{\alpha_{t-1}}} \\
    & = \sum_{{i_t = 1}}^{|V|} \sum_{{\alpha_t = 1}}^{R} f(x^{(t)})_{i_t}\etG_{i_t\alpha_{t}\alpha_{t-1}} h^{(t-1)}_{\textrm{RAC}_{\alpha_{t-1}}} \\
\end{split}
\end{equation}

where $\tG = \mW^{hx} \odot \mW^{hh}$. In this case, the hidden state of TTLM in Eq. \ref{eq: ttlm cell} is equal to the hidden state of RACs in Eq. \ref{eq: RAC_element},  as shown in Fig. \ref{fig:rnns}b.  Similarly,  if Eq. \ref{eq: ttlm cell} is accompanied with an activation function $\phi$,  Eq. \ref{eq: ttlm cell}  is equal to the hidden state of MI-RNNs in Eq. \ref{eq: mirnn} as shown in Fig. \ref{fig:rnns}c. 

\end{proof}

Given Claim \ref{cm: G and T} and \ref{cm: G and RACs}, the three models shall be simulated by TTLM with a nonlinear activation function and we leave finding a theoretical proof of this conjecture to a future work.

\section{Comparison with related work on experiment datasets}
\label{sec:datasets}
\begin{table}[t]
    \centering
    \begin{tabular}{l|c|c|c|c}
    \toprule
        Dataset & Data type & Vocab  & $N$ & Real-world language  \\
         \midrule
        Tomita grammars  \cite{tomita1982dynamic}  & Disc. & 2  & 10k & $\times$ \\ 
        Motzkin grammar \citep{alexander2021exact} & Disc. & 3 & 10k   &$\times$ \\
         Email addresses \citep{radev2008clair} & Disc. & $\leq 256$ & 4k & $\times$  \\
        \midrule
        POWER  \cite{Dua:2019} & Cont. & - & 1659k & $\times$ \\
        GAS \cite{fonollosa2015reservoir} & Cont. & - & 852k &  $\times$ \\
        HEPMASS \cite{baldi2016parameterized} & Cont. & - & 315k & $\times$  \\
        MINIBOONE \cite{roe2005boosted} & Cont. & -  & 29k & $\times$  \\
        BSDS300 \citep{martin2001database}  & Cont. & - & 1000k &  $\times$ \\
         
        \midrule
        PTB \citep{marcinkiewicz1994building} & Disc. & 10k  & 31k & $\surd$ \\
        WikiText-2  \citep{merity2016pointer} & Disc. & 30k & 73k &  $\surd$ \\  
        
 \bottomrule
    \end{tabular}
    \caption{The column "datasets" means the training datasets. The first three are used in  \citep{miller2021tensor}, the following five are used in  \cite{novikov2021tensor}, and the last two are used in our paper.  The column "$N$" denotes the number of training examples. The column "Vocab" denotes the number of types. The value "Cont." denotes continuous variable, while the "Disc." denotes  discrete variables. }
    \label{tab:datasets}
\end{table}

\end{document}

%% file: math_commands.tex
\usepackage{amsmath,amsfonts,bm}

\def\eqref#1{equation~\ref{#1}}

\def\1{\bm{1}}

\def\vb{{\bm{b}}}

\def\vf{{\bm{f}}}

\def\vh{{\bm{h}}}

\def\vv{{\bm{v}}}

\def\vy{{\bm{y}}}

\def\mG{{\bm{G}}}

\def\mM{{\bm{M}}}

\def\mP{{\bm{P}}}

\def\mV{{\bm{V}}}
\def\mW{{\bm{W}}}

\DeclareMathAlphabet{\mathsfit}{\encodingdefault}{\sfdefault}{m}{sl}
\SetMathAlphabet{\mathsfit}{bold}{\encodingdefault}{\sfdefault}{bx}{n}
\newcommand{\tens}[1]{\bm{\mathsfit{#1}}}
\def\tA{{\tens{A}}}
\def\tB{{\tens{B}}}
\def\tC{{\tens{C}}}
\def\tD{{\tens{D}}}

\def\tG{{\tens{G}}}

\def\tP{{\tens{P}}}

\def\tT{{\tens{T}}}

\def\tV{{\tens{V}}}
\def\tW{{\tens{W}}}
\def\tX{{\tens{X}}}
\def\tY{{\tens{Y}}}

\newcommand{\etens}[1]{\mathsfit{#1}}

\def\etA{{\etens{A}}}

\def\etC{{\etens{C}}}
\def\etD{{\etens{D}}}
\def\etE{{\etens{E}}}

\def\etG{{\etens{G}}}

\def\etT{{\etens{T}}}

\def\etX{{\etens{X}}}
\def\etY{{\etens{Y}}}

\newcommand{\softmax}{\mathrm{softmax}}



%% file: sections/Preliminaries.tex
\section{Preliminaries}
\label{sec:preliminaries}
We briefly recapitulate basic notions and notations\footnote{Most of the notations here  follow the textbook \textit{Deep Learning}
\cite{goodfellow2016deep}.}; full technical introductions can be found in standard textbooks \cite{BI20221, 10.5555/1643460}.


\paragraph{Notation.} For the purposes of this paper, every \textit{tensor} $\tA$  is a multidimensional array of elements (called \emph{components}) of $\mathbb{R}$, each denoted by its integer coordinates in the array; e.g., for a two-dimensional array, the component at position $i,j \in \mathbb{N}$ is denoted $\etA_{ij}$.
 The \emph{order} of a tensor is how many indices it has (e.g., a vector $\vv$ is a first-order tensor, a matrix $\mM$ is a second-order tensor, etc.). The \emph{dimension} of a tensor refers to the number of values that a particular index (or so-called \textit{mode}) can take, e.g., the dimension of $\tB \in \mathbb{R}^{I_1 \times I_2 \times I_3}$ is $I_1 \times I_2 \times I_3$. 

\paragraph{Tensor Product \cite{cohen2016expressive}.} For two tensors $\tC \in \mathbb{R}^{I_1 \times \cdots \times I_j}$ (order $j$) and $\tD \in \mathbb{R}^{I_{j+1}, \times \cdots \times I_{j+k}}$ (order $k$), their \textit{tensor product} is denoted by $\otimes$ and return a tensor $\etE_{i_1 \cdots i_{j+k}} = \etC_{i_1...i_j} \cdot \etD_{i_{j+1} \cdots i_{j+k}}$ (order $j+k$).  Notice that in the case $j = k = 1$, the tensor product reduces to an outer product between vectors.

\paragraph{Generalized Inner Product  \citep{kossaifi2020tensor}.}
\label{Generalized inner product}
For two tensor $\tX,\tY\in\mathbb{R}^{I_1\times I_2 \times \cdots \times I_N}$ of the same size, their \textit{inner product} is defined as $\langle \tX,\tY \rangle=\sum_{i_1=1}^{I_1}\sum_{i_2=1}^{I_2} \cdots \sum_{i_N=1}^{I_N} \etX_{i_1,i_2,...,i_N} \etY_{i_1,i_2,...,i_N}$. For two tensors $\tX \in \mathbb{R}^{I_1\times I_2 \times \cdots \times I_N \times I_x}$ and $\tY \in \mathbb{R}^{I_1\times I_2\times \cdots I_N \times I_y}$ sharing $N$ modes of the same size, the “generalized inner product” is calculated as
\begin{align} 
    \langle \tX,\tY \rangle_N =\sum_{i_1=1}^{I_1}\sum_{i_2=1}^{I_2}\cdot\cdot\cdot\sum_{i_N=1}^{I_N} \etX_{i_1,i_2,...,i_N} \etY_{i_1,i_2,...,i_N} \nonumber
\end{align}
with $\langle \tX,\tY \rangle_N \in \mathbb{R}^{I_x \times I_y}$.

\begin{figure*}[t]
\centering
\includegraphics[scale=0.45]{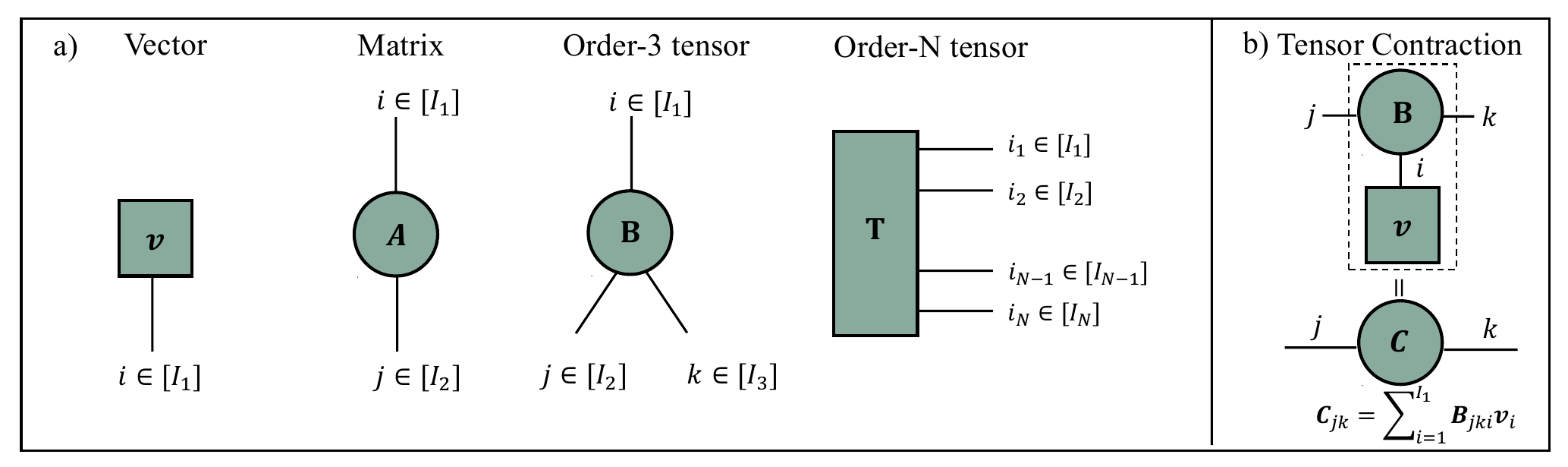}
\caption{A quick introduction to \textit{tensor diagram notation}. There are two rules of tensor diagrams: (1) tensors are notated by solid shapes with a number of 'legs'  corresponding to their indices; (2) connecting two index lines implies a \textit{contraction} or summation over the connected indices. In this paper, we augment our equations with these diagrams to make them easier to understand. }
\label{fig: tensor_diagram}
\end{figure*}